\documentclass{article}
\usepackage{arxiv}
\usepackage[utf8]{inputenc} % allow utf-8 input
\usepackage[T1]{fontenc}    % use 8-bit T1 fonts
\usepackage{hyperref}       % hyperlinks
\usepackage{url}            % simple URL typesetting
\usepackage{booktabs}       % professional-quality tables
\usepackage{amsfonts}       % blackboard math symbols
\usepackage{nicefrac}       % compact symbols for 1/2, etc.
\usepackage{microtype}      % microtypography
\usepackage{lipsum}		    % Can be removed after putting your text content
\usepackage{graphicx}
\usepackage[numbers]{natbib}
\usepackage{doi}

\usepackage{amsmath}
\usepackage{makecell}
\usepackage{booktabs}
\usepackage[table]{xcolor}
\usepackage{caption}
\usepackage{listings}
\usepackage{xcolor}
\usepackage{amsthm}
\usepackage{pythonhighlight}

\usepackage{mathtools} 
\DeclarePairedDelimiter{\norm}{\lVert}{\rVert}

\usepackage{amsthm}
\newtheorem{theorem}{Theorem} % Define a theorem environment
\newtheorem{lemma}[theorem]{Lemma}

\theoremstyle{definition}

\usepackage{algpseudocode} % for pseudocode
\usepackage{algorithm}

\usepackage{pythonhighlight}

\title{An Iterative Geometric Approach to Optimizing Separating Hyperplanes}

\author{ \href{https://orcid.org/0000-0001-6410-8895}{\includegraphics[scale=0.06]{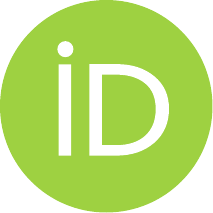}\hspace{1mm}\'Akos Hajnal}\textsuperscript{a,b}\\
	\textsuperscript{a}Laboratory of Parallel and Distributed Systems\\
        Institute for Computer Science and Control (SZTAKI) \\
        \textsuperscript{b}John von Neumann Faculty of Informatics \\
        Obuda University \\
	Budapest, Hungary \\
	\texttt{akos.hajnal@sztaki.hu} \\
}

\renewcommand{\undertitle}{}
\renewcommand{\headeright}{}
\hypersetup{
pdftitle={An Iterative Geometric Approach to Optimizing Separating Hyperplanes},
pdfauthor={Akos Hajnal},
pdfkeywords={Optimizing, Separating Hyperplanes}
}

\begin{document}
\maketitle

\begin{abstract}
Given a binary-labeled linearly separable dataset, and the objective is to compute the maximum-margin separating hyperplane, also known as the hard-margin Support Vector Machine (SVM) classifier. This paper investigates whether, if given an initial separating hyperplane, can it be exploited to reach this unique optimum more efficiently. We present a geometric approach that gradually improves the alignment of the hyperplane, starting from an initial separating hyperplane, while preserving separation and continuously increasing its margin until convergence to the global optimum. At each iteration, the method considers only local information, namely the current active set, and aims to re-align the hyperplane according to the optimal separating hyperplane of this reduced subset. Consequently, the original convex quadratic optimization problem is addressed through a sequence of smaller subproblems. The paper presents the algorithm in detail, together with a preliminary experimental evaluation and several theoretical findings. The results suggest that, when an initial separating hyperplane is available, the proposed method can be competitive on larger datasets and, in some cases, can outperform state-of-the-art approaches that solve the optimization problem directly.
\end{abstract}


\setcitestyle{square}

\section{Introduction}

Given a binary-labeled, linearly separable  dataset, the \textit{maximum-margin} separating hyperplane is the hyperplane that maximizes the distance to the nearest samples while maintaining the separation of samples belonging to the different classes. This hyperplane is unique and also known as the \textit{optimal} separating hyperplane or the hard-margin Support Vector Machine (SVM) classifier introduced by Vladimir Vapnik \cite{cortes1995support}. 

For large datasets, even testing linear separability (i.e., feasibility) can be computationally demanding, which can be expressed as a linear program. Computing the maximum-margin separator leads to a convex quadratic optimization problem, which is commonly solved using Sequential Minimal Optimization (SMO) \cite{platt1998sequential} or other convex optimization techniques \cite{boyd2004convex}. Solver implementations are available in tools such as CVXPY \cite{diamond2016cvxpy}, LIBSVM \cite{chang2011}, and LIBLINEAR \cite{fan2008liblinear}.

This paper investigates whether, if given an initial separating hyperplane, which is not optimal, can it be iteratively refined through a series of gradual re-alignments into the optimal separating hyperplane -- while maintaining the separation and continuously increasing the margin. 

To illustrate the core idea, consider an initial separating hyperplane $H_1$ together with its closest oppositely labeled samples, $p_1$ and $n_1$, as shown in Figure 1. First, $H_1$ is aligned to lie equidistantly between $p_1$ and $n_1$. If this hyperplane is not already optimal, its normal vector $w_1$ is rotated toward the direction of $\overrightarrow{n_1  p_1}$ about the midpoint of $p_1$ and $n_1$, which serves as the pivot. This re-alignment continuously increases the margin with respect to both $p_1$ and $n_1$. The rotation continues until either the normal vector becomes parallel to $\overrightarrow{n_1  p_1}$, yielding the optimal separator, or the rotation is blocked by another sample, $p_2$ in the figure, resulting in the adjusted hyperplane $H_2$.

In the latter case, the newly encountered blocking sample $p_2$ is incorporated into the 'active set', samples closest to the current hyperplane, and a new 'preferred' rotation direction is selected, namely $\overrightarrow{n_1  p_2}$. The normal vector of $H_2$ is then rotated toward this new direction until either another blocking sample is encountered or the maximum-margin separating hyperplane is reached. In the example, no further sample blocks the rotation, and the process terminates at the maximum-margin separator $H_3$.

\begin{figure*}[t]
\centering
\includegraphics[width=0.35\linewidth]{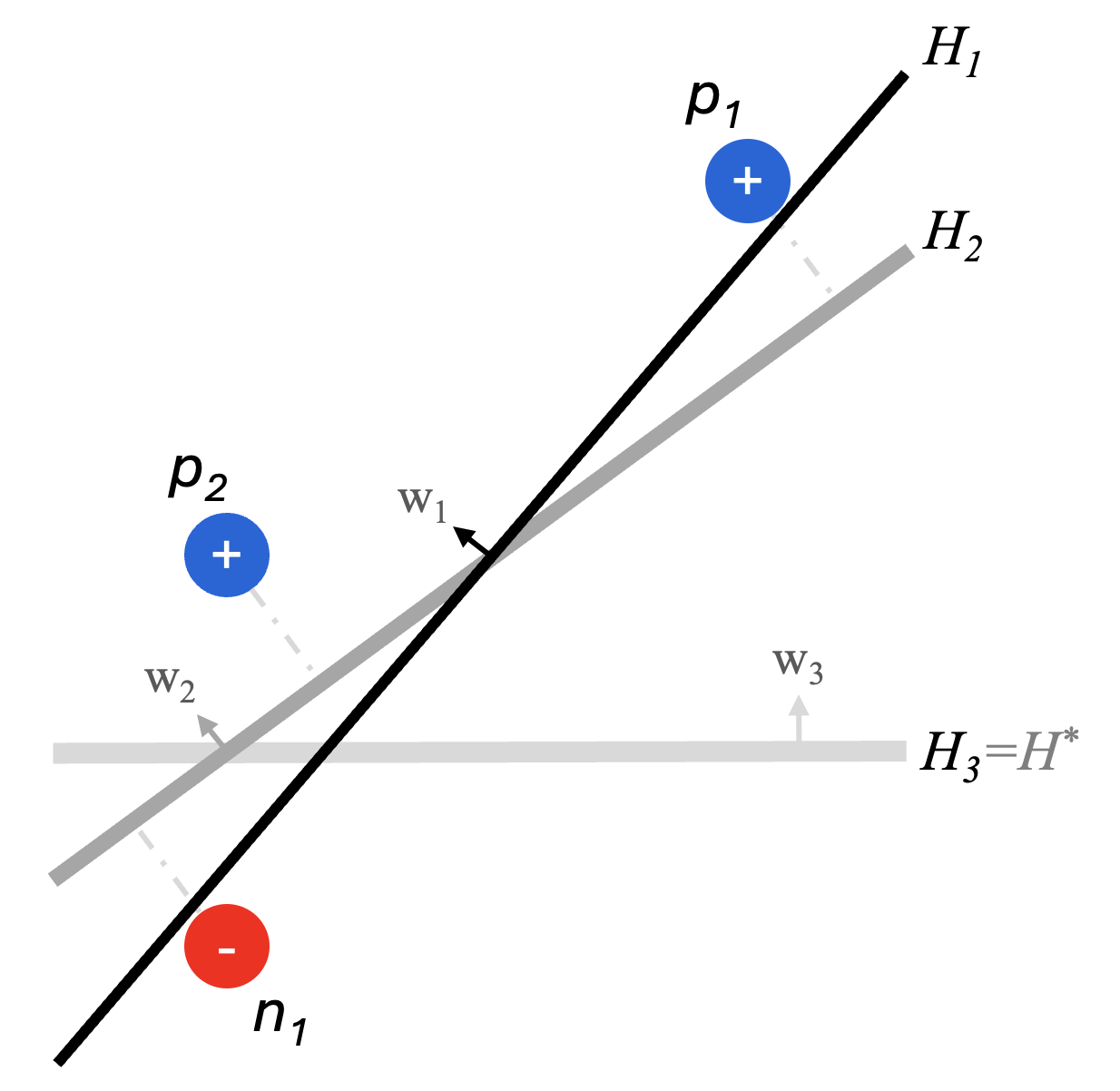}
\caption{Illustration of the method: starting from an initial separating hyperplane $H_1$, the margin is iteratively increased by re-aligning the hyperplane, progressing through $H_2$ and finally reaching the maximum-margin hyperplane $H_3$.}\label{fig1}
\end{figure*}

As shown later, the preferred rotation direction is chosen as the normal vector of the maximum-margin separating hyperplane of the current active set. Intuitively, this choice was motivated by the fact that the optimality of a separating hyperplane is determined by the samples closest to it. If the hyperplane is already optimal with respect to these samples, then it is globally optimal (formalized by Lemma 1 in the Appendix). Otherwise, the projections of the two classes onto the current hyperplane form disjoint convex hulls (formalized by Lemma~3). To reduce the separation between these projected sets, the most natural direction is the one connecting their closest points, which can be achieved by rotating the normal vector of the current hyperplane toward this direction. This direction is aligned with the normal vector of the maximum-margin separating hyperplane of the current active set.

The proposed approach can be beneficial in cases where finding a feasible separating hyperplane is computationally easier than directly computing the optimal one, or where such a hyperplane is already available. Beyond formalizing and evaluating the proposed method, we present some theoretical findings that provide additional insight into maximum-margin hyperplane optimization, including the correspondence between local and global optima and another geometric interpretation of the Karush--Kuhn--Tucker (KKT) conditions in the context of SVMs \cite{cortes1995support, boyd2004convex}.

The main contributions of the paper are as follows:
\begin{itemize}
    \item an iterative technique to optimize separating hyperplanes,
    \item experimental evaluation of the proposed method,
    \item theoretical findings on maximum-margin hyperplane optimization.
\end{itemize}

In the next section, we introduce the background and terminology. Section 3 presents the hyperplane optimization technique. Section 4 evaluates the proposed method. Section 5 compares the method with related approaches, finally, Section 6 concludes the paper. Theoretical contributions are formalized in the Appendix.

\section{Background}
If given a linearly separable dataset $X \in \mathbb{R}^{n \times d}$ consisting of $n$ samples (rows of $X$) and $d$ features (columns of $X$), together with binary class labels $y \in \{+1,-1\}^n$ associated with each sample, hyperplane 

\begin{equation}
H : w^T x + b = 0
\end{equation}

(strictly) \emph{ separates} the dataset, if

\begin{equation}
y_i(w^T x_i + b) > 0,\qquad i=1,\ldots,n.
\end{equation}

The \emph{positive side} of the hyperplane is determined by the orientation of its normal vector and consists of the half-space defined by $w^Tx+b>0$. The \emph{negative side} is defined analogously by $w^Tx+b<0$.

The \emph{maximal-margin separating hyperplane} $H^* : w^{*T}x + b^* = 0$ is the separator hyperplane having the maximal geometric margin (Euclidean distance) to the nearest samples:

\begin{equation}
(w^*, b^*) =
\arg\max_{w,b}
\min_{i=1,\ldots,n}
\frac{y_i(w^T x_i + b)}{\norm{w}}.
\end{equation}

Using the canonical hard-margin SVM formulation, $w^*$ and $b^*$ are the solutions of the (quadratic) optimization problem: 

\begin{equation}
\begin{aligned}
\min\limits_{w, b} \quad & \frac{1}{2} \norm{w}^2 \\ 
\textrm{s.t.} \quad & y_i (w^{T} x_i + b) \geq 1, \quad i=1,\ldots,n. 
\end{aligned}
\end{equation}

The maximal-margin separating hyperplane $H^*$ is unique, also referred to as the \emph{optimal} separating hyperplane, and is determined by a set of \emph{support vectors} $S \subseteq X$ \cite{cortes1995support}. All support vectors lie at the same geometric distance from $H^*$, called the \textit{margin}:

\[
\gamma^*=\frac{y_s w^{*T}s+b^*}{\norm{w^*}}, \quad s \in S .
\]

The \emph{positive} and \emph{negative margin hyperplanes} are hyperplanes parallel to $H^*$, located at a distance $\gamma^*$ from $H^*$ on the positive and negative sides, respectively. All support vectors lie on these margin hyperplanes, equidistant from $H^*$. No other samples can occur within the region (\textit{slab}) bounded by them, thus support vectors are the samples closest to $H^*$.

In the simplest case, there is exactly one support vector from each class, and $H^*$ is the perpendicular bisector of the segment connecting these two samples. At the opposite extreme, every sample in the dataset can be a support vector. In practice, the number of support vectors is typically much smaller than the total number of samples.

We adopt the convention that every sample lying on either margin hyperplane of the optimal separating hyperplane is regarded as a support vector. This definition is slightly broader than the one induced by the Karush--Kuhn--Tucker (KKT) conditions \cite{cortes1995support}. In particular, some support vectors may be redundant, in the sense that they are assigned zero KKT multipliers in a valid optimal solution, while the remaining support vectors receive positive coefficients. Nevertheless, throughout this paper, all samples lying on the margin hyperplanes are referred to as support vectors.

The \emph{active set} of a separating hyperplane $H: w^Tx+b=0$ consists of the samples from both classes that are closest to $H$. Assuming that $H$ is positioned equidistantly between the nearest oppositely labeled samples (by adjusting the bias while preserving its orientation), the \emph{positive} and \emph{negative margin hyperplanes} are defined as the hyperplanes parallel to $H$ that pass through the positive and negative active samples, are denoted by $H^+$ and $H^-$, respectively.

\section{An Iterative Separating Hyperplane Optimization Method}

The algorithm takes two inputs: the dataset $X$ with corresponding labels $y$, and an initial separating hyperplane $H_1$, aligned to be equidistant from the nearest samples. It then performs a sequence of iterations, each transforming the current separating hyperplane $H_i: w_i^T x + b_i = 0$ into $H_{i+1}: w_{i+1}^T x + b_{i+1} = 0$, with refined alignment and an increased margin. An iteration consists of four steps: (1) First, it calculates the active set $A_i$ of $H_i$. (2) Then, the appropriate re-alignment of $H_i$ is chosen, consisting of a rotation of its normal vector $w_i$ about a suitably chosen pivot point, based on the current active set $A_i$. (3) It is followed by determining the maximal admissible rotation in this direction, considering all samples in $X$. (4) Finally, the normal vector and bias are updated to obtain $w_{i+1}$ and $b_{i+1}$  accordingly, yielding the next separating hyperplane $H_{i+1}$. If no 'blocking samples' encountered in step (3), we reached the optimal solution, otherwise the calculation continues with the next iteration.

The pseudocode is shown in Algorithm \ref{algorithm}. The main steps are described in detail in the following subsections.

\begin{algorithm}[t]
\caption{Iterative Separating Hyperplane Optimization}
\label{algorithm}
\begin{algorithmic}[1]
\Require
    \Statex Linearly separable dataset $X=\{(x_i,y_i)\}_{i=1}^{n}$, $x_i \in \mathbb{R}^{d}$, $y_i\in\{-1,+1\}$,
    \Statex Separating hyperplane $H_1:w_1^Tx+b_1=0$, mid-aligned between the closest oppositely labeled samples.
\Ensure 
    \Statex Maximum-margin separating hyperplane $H^*:w^{*T}x+b^*=0$.
\State $i\gets1$
\While{true}
    \State Determine the active set $A_i$ of $H_i$
    \State Compute the optimal separating hyperplane $H_i^*:w_i^{*T}x+b_i^*=0$ of $A_i$
    \If{$H_i = H_i^*$}
        \State \Return $H_i$
    \EndIf
    \State Compute the maximum admissible interpolation parameter $\alpha$
    \State Update
    \[
    w_{i+1}\gets(1-\alpha)w_i+\alpha w_i^*, \quad b_{i+1}\gets -w_{i+1}^Tm_i,
    \]
    \[
    H_{i+1}\gets w_{i+1}^Tx + b_{i+1} =0,
    \]
    \quad\, where $m_i$ is the selected pivot (midpoint of a support vector pair of $H_i^*$)
    \State $i\gets i+1$
\EndWhile
\end{algorithmic}
\end{algorithm}

\subsection{Choosing a proper direction of re-alignment}

We are given a separating hyperplane $H_i$ and its active set $A_i \subseteq X$, samples closest to and equidistant from $H_i$. Our first goal is to decide whether $H_i$ is optimal with respect to $A_i$ or not. If not, our next goal is to choose a proper re-alignment of $H_i$ that increases the distance to all samples in $A_i$.

To determine whether $H_i$ is optimal with respect to the current active set (locally), we first apply an existing optimization method exclusively to the samples in $A_i$ to compute the optimal separating hyperplane $H_i^* : w_i^{*T}x + b_i^* = 0$ of $A_i$ and the corresponding set of support vectors $S_i \subseteq A_i$.

If $H_i$ coincides with $H_i^*$ (local optimum) and $H_i$ is a separating hyperplane (globally), then, as formally established by Lemma~1 in the Appendix, $H_i$ is the globally optimal separating hyperplane. Consequently, the algorithm can terminate.

If $H_i$ is not optimal, then there exists a re-alignment of $H_i$ that increases the margin with respect to $A_i$, which corresponds to rotating the normal vector $w_i$ toward $w_i^*$ about an appropriately chosen pivot. 

To formalize the rotation of a hyperplane from $H_i$ to $H_i^*$, we introduce a parameterized normal vector

\begin{equation}\label{wia}
    w_i(\alpha) = (1 - \alpha) w_i + \alpha w_i^*, \quad 0 \le \alpha \le 1,
\end{equation}

where the pivot is chosen as the midpoint $m$ of an arbitrarily selected oppositely labeled support vector pair $s_p, s_n \in S_i$, which yields the bias

\begin{equation}
    b_i(\alpha) = -w_i(\alpha)^T\left(\frac{s_p+s_n}{2}\right).
\end{equation}

The hyperplane is then given by

\begin{equation}
 H_i(\alpha): w_i(\alpha)^Tx + b_i(\alpha) = 0.
\end{equation}

Notice that as $\alpha$ increases from $0$ to $1$, hyperplane $H_i(\alpha)$ transforms from $H_i = H_i(0)$ to $H_i^* = H_i(1)$. First, we show that all support vectors remain on the positive and negative margin hyperplanes throughout the entire rotation, while non-support active samples move farther away from these boundaries, which behavior preserves proper separation of $A_i$ by $H_i(\alpha)$.

Let us introduce two hyperplanes $H_i^+(\alpha)$ and $H_i^-(\alpha)$, parallel to and equidistant from $H_i(\alpha)$, where $H_i^+(\alpha)$ passes through $s_p$ and $H_i^-(\alpha)$ passes through $s_n$:

\begin{equation}\label{hip}
\begin{aligned}
H_i^+(\alpha): w_i(\alpha)^T x - w_i(\alpha)^T s_p = 0,\\
H_i^-(\alpha): w_i(\alpha)^T x - w_i(\alpha)^T s_n = 0.
\end{aligned}
\end{equation}

First, we show that any positive support vector $s \in S_i$ remains on $H_i^+(\alpha)$, independently of the value of $\alpha$. Rearranging Equation~\ref{hip} and substituting $w_i(\alpha)$ from Equation~\ref{wia}, we obtain

\begin{equation}\label{hipr}
\begin{aligned}
H_i^+(\alpha): w_i(\alpha)^T x - w_i(\alpha)^T s_p \\
= w_i(\alpha)^T (x - s_p) \\
= \big((1 - \alpha) w_i + \alpha w_i^*\big)^T (x - s_p) \\
= (1 - \alpha) w_i^T (x - s_p) + \alpha w_i^{*T} (x - s_p) = 0.
\end{aligned}
\end{equation}

Since both $s$ and $s_p$ are active points, the vector from $s$ to $s_p$ is orthogonal to the normal vector $w_i$. Moreover, since $s$ and $s_p$ are both support vectors, the same vector is also orthogonal to $w_i^*$. Substituting $x = s$ into Equation~\ref{hipr} yields zero for both inner products $w_i^T (x - s_p)$ and $w_i^{*T} (x - s_p)$, implying that $s$ remains on hyperplane $H_i^+(\alpha)$ for any value of $\alpha$. The same argument applies analogously to negative support vectors respectively, implying that all such samples remain on the hyperplane $H_i^-$, throughout the entire rotation. (Note that this statement is true for any sample in $A_i$ that lies on either margin hyperplane of $H_i^*$; therefore, it remains valid under the broader definition of support vectors.)

Now consider a positive active sample $p$ that is not a support vector, and examine its position relative to the hyperplane $H_i^+(\alpha)$. Since both $p$ and $s_p$ are active points of $H_i$, $w_i^T(p-s_p)=0$. Substituting $x=p$ into Equation~\ref{hipr} yields $H_i^+(\alpha)=\alpha w_i^{*T}(p-s_p)$. $p$ is not a support vector, it lies strictly beyond the positive margin hyperplane of $H_i^*$; in contrast, $s_p$ lies on that positive margin hyperplane. It follows that $w_i^{*T}(p-s_p)>0$. Consequently, $H_i^+(\alpha)>0$ for every $\alpha>0$, hence, all non-support positive actives get farther from the positive margin hyperplane during the rotation (with a distance increasing linearly with $\alpha$). The same argument applies analogously to negative non-support active samples, where the value of $H_i^-(\alpha)$ becomes negative, indicating that these samples get below the negative margin hyperplane, and so also leave the slab during the rotation.

Since all the support vectors keep lying on the positive and negative margin hyperplanes of $H_i(\alpha)$ for every value of $\alpha$ while the slab bounded by them remains free of samples of $A_i$, $H_i(\alpha)$ remains a separating hyperplane for $A_i$. 

The margin of $H_i(\alpha)$ is half the distance between these margin hyperplanes. To establish the continuous increase of this margin, let $\theta(\alpha)$ denote the angle between the normal vectors $w_i(\alpha)$ and $w_i^*$. Applying Lemma 2 in the Appendix, we obtain:

\begin{equation}\label{gamma_alpha_and_gammastar}
\begin{aligned}
\gamma_i(\alpha)=\gamma_i^*\cos\theta(\alpha).
\end{aligned}
\end{equation}

Since $w_i(\alpha)$ rotates continuously from $w_i$ toward $w_i^*$, the angle $\theta(\alpha)$ decreases continuously as $\alpha$ increases.  Consequently, $\cos\theta(\alpha)$ increases continuously, and, by Equation \ref{gamma_alpha_and_gammastar}, so does the margin $\gamma_i(\alpha)$. Therefore, the margin increases strictly monotonically throughout the entire rotation, until it reaches the optimal value $\gamma_i^*$ at $\alpha=1$.

To summarize, the rotation of the current normal vector toward the normal vector of the optimal separating hyperplane around a properly selected pivot point has the following properties:
\begin{itemize}
\item it preserves the strict separation of the active set $A_i$;
\item it increases the margin strictly monotonically.
\end{itemize}

\subsection{Determining the extent of re-alignment}

As shown, rotation about the selected pivot increases the margin with respect to the current active set and may eventually lead to the optimal alignment. However, samples outside the current active set may encounter the margin hyperplanes before the local optimum is reached, thereby limiting further rotation in that direction. Continuing the rotation beyond this point would violate the requirement that the slab remains free of samples.

The objective of this step is thus to determine the maximum admissible rotation of the margin hyperplanes about the selected positive and negative support vectors, defining the slab, without crossing any other samples. For this reason, we compute the maximum rotation of the positive margin hyperplane until it reaches a new positive sample and, analogously, the maximum rotation of the negative margin hyperplane until it reaches a new negative sample. The smaller of these two values determines the maximum admissible rotation.

Since the current hyperplane separates all samples, the positive margin hyperplane cannot encounter a negative sample before the negative margin hyperplane does, and vice versa. Therefore, only positive samples located between the positive margin hyperplanes of $H_i$ and $H_i^*$, and negative samples located between the negative margin hyperplanes of $H_i$ and  $H_i^*$, can limit the admissible rotation. Given the samples outside the active set, the parametric normal vector $w_i(\alpha)$, and the pivots $s_p$ and $s_n$ of the margin hyperplanes, the maximal admissible value of $\alpha$ can be determined. 

Considering a positive (non-active) sample $p$ and the corresponding positive margin hyperplane $H_i^+(\alpha)$, the value $\alpha_{max}^+$ at which $H_i^+(\alpha)$ first crosses $p$ is when the vector $p-s_p$ becomes orthogonal to the interpolated normal vector $w_i(\alpha)$, i.e.,

\begin{equation}\label{max_alpha}
\begin{aligned}    
    w_i(\alpha_{max}^+)^T(p - s_p) = 0 \\
    ((1 - \alpha_{max}^+) w_i + \alpha_{max}^+ w_i^*)^T(p - s_p) = 0.
\end{aligned}
\end{equation}

Rearranging yields

\begin{equation}
    \alpha_{max}^+(p) = \frac{w_i^T(p - s_p)}{w_i^T(p - s_p)-w_i^{*T}(p - s_p)}. 
\end{equation}

The corresponding value $\alpha_{max}^-(n)$ for a negative sample $n$ is obtained analogously using the negative pivot $s_n$. 

Samples that would reach the margin only after extending the interpolation beyond $w_i^*$ do not limit the current rotation and produce values $\alpha_{max}>1$. The maximum admissible interpolation parameter $\alpha_{actual}$ is therefore determined by the first sample that reaches either margin hyperplane:

\begin{equation}
\alpha_{actual} = \min\left( 1, \min\limits_{p \in X: y_p=+1} \alpha_{max}^+(p), \min\limits_{n \in X: y_n=-1} \alpha_{max}^-(n) \right),
\end{equation}

where $y_i\in\{+1, -1\}$ denotes the label associated with sample $i$.

Notice that all candidate values of $\alpha_{max}$ can be computed simultaneously using matrix operations.

If $\alpha_{actual}=1$, the re-alignment reaches the optimal direction, and the computation can stop. Otherwise, the rotation is applied with $\alpha_{actual}$, while keeping the separating hyperplane passing through the selected pivot point, yielding the next separating hyperplane for the subsequent iteration:

\begin{equation}
\begin{aligned}    
    w_{i+1} = (1 - \alpha_{actual}) w_i + \alpha_{actual} w_i^*,\\
    b_{i+1} = -\frac{1}{2}w_{i+1}^T(s_p+s_n),\\
    H_{i+1} : w_{i+1}^Tx + b_{i+1} = 0.
\end{aligned}
\end{equation}

Notice that the norm of the interpolated normal vector is not constrained during the rotation, therefore, after each update, the bias is adjusted to keep the hyperplane passing through the selected pivot. Moreover, the active set in the next iteration will include the newly encountered blocking sample(s), while some samples from the current active set (non-support vectors) no longer lie on the updated margin hyperplanes, thus leave this set.

\subsection{Complexity}

Let $n$ denote the number of samples, $d$ the number of features, $a_i=|A_i|$ the size of the active set in iteration $i$, and let $iter$ denote the number of iterations until convergence.

In each iteration:
\begin{enumerate}
    \item Determining the \emph{active set} requires computing the signed distances of all samples from the current hyperplane, which takes $\mathcal{O}(nd)$ time.
    \item Computing the \emph{optimal separating hyperplane} for the active set $a_i$ depends on the employed SVM solver. Denoting the running time of the solver on $a_i$ samples with $d$ features by $T_{\mathrm{SVM}}(a_i,d)$, this step has complexity
    \[
    \mathcal{O}(T_{\mathrm{SVM}}(a_i,d)).
    \]
        
    \item Computing the \emph{maximum admissible interpolation parameter} $\alpha$ requires evaluating all samples against the interpolated margin hyperplanes. Using matrix operations, this step requires $\mathcal{O}(nd)$ time.
    \item \emph{Updating} the hyperplane involves operations on the normal vector and bias, requiring $\mathcal{O}(d)$ time.
\end{enumerate}

Hence, the running time of a single iteration is 
\[
\mathcal{O}(nd+T_{\mathrm{SVM}}(a_i,d)).
\]

If the algorithm converges after $iter$ iterations, the total running time is
\[
\mathcal{O}(\sum_{i=1}^{iter} (nd+T_{\mathrm{SVM}}(a_i,d))).
\]

% If the active-set size is bounded by a constant $k$, the complexity becomes
% \[
% \mathcal{O}(iter(nd+T_{\mathrm{SVM}}(k, d))).
% \]

The most expensive step is the computation of the optimal separating hyperplane of the current active set. We do not specify the SVM implementation, the corresponding computational complexity depends on the chosen solver. Typically $|A_i| \ll n$. The existence of a finite upper bound on $iter$, is not proven in this paper and remains an open question.

% For example, linear SVM solvers based on LIBLINEAR typically exhibit near-linear scaling in the number of samples and features, while SMO-type kernel SVM solvers may have substantially higher complexity, with commonly cited worst-case bounds between $\mathcal{O}(a_i^2)$ and $\mathcal{O}(a_i^3)$.
% We note that although the margin increases strictly monotonically and is bounded above, the finite-step termination of the method, i.e., the existence of a finite bound on $iter$, is not proven in this paper and remains an open question.

\section{Experimental Results}

The experiments were carried out in the Google Colaboratory (Colab) \cite{google-colab} hosted Jupyter Notebook environment. We selected the 'CPU' runtime\footnote{CPU: Intel(R) Xeon(R) CPU @ 2.20GHz (model: 79, family: 6), RAM: 12.7 GB. No GPU acceleration was utilized.} and implemented the proposed solution in Python using NumPy and CVXPY.\footnote{Python v3.12.13, NumPy v2.0.2, CVXPY v1.6.7, OSQP v1.1.3, CLARABEL v0.11.1.}

We used the MNIST handwritten digits dataset \cite{lecun1998a}, consisting of grayscale images with a resolution of 28x28 pixels, represented as 784-dimensional feature vectors. Each experiment considers a pair of digits; therefore, only samples belonging to these two digits are retained, with one digit labeled as the positive class and the other as the negative class (one-vs-one classification). We selected 10 linearly separable digit pairs \cite{hajnal2026linearseparabilitymnisthandwritten} and performed experiments in both the training set (60,000 samples) and the test set (10,000 samples). Thus, a total of 20 experiments were conducted, each repeated 10 times; the reported execution times represent the averages over these 10 runs.

CVXPY was chosen as the baseline for computing maximum-margin separating hyperplanes because it provided the most accurate results. CVXPY used solver the OSQP solver \cite{osqp} internally, and the optimization problem was formulated as follows:

\begin{lstlisting}[language=Python,basicstyle=\small]
import cvxpy as cp
...
w = cp.Variable(number_of_features)
b = cp.Variable()
constraints = [cp.multiply(y, X @ w + b) >= 1] 
problem = cp.Problem(cp.Minimize(cp.sum_squares(w)), constraints)
problem.solve()    
...
\end{lstlisting}

where \texttt{X} and \texttt{y} contain the samples and their corresponding labels. The normal vector and bias of the optimal separating hyperplane are returned in the variables \texttt{w} and \texttt{b}, respectively. Depending on the choice of \texttt{X}, it computes either a local or a global optimum, respectively.

To obtain a separating hyperplane of the dataset (a feasible but not necessarily optimal solution), we replaced the optimization objective with \texttt{cp.Minimize(0)}. In these cases, CVXPY employed the CLARABEL solver \cite{Clarabel_2024}.

The objective of the experiments was to compare the computation time of the optimal solution for the entire dataset with the time required for computing a feasible solution followed by the proposed iterative refinement. The results are shown in Tables~\ref{table1}, where the columns are defined as follows:
\begin{itemize}
    \item \textbf{Exp.}: The label of the experiment. The values `train' and `test' refer to the training and test datasets, respectively; the numbers identify the digits being separated.
    \item \textbf{Samples}: The number of samples in the dataset.
    \item \textbf{Baseline}: The time required to compute the optimal separating hyperplane for the entire dataset using CVXPY (baseline), in seconds.
    \item \textbf{Current}: The total time required to compute a feasible separating hyperplane and then iteratively optimize it (the proposed method), in seconds.
    \item \textbf{Ratio}: The speedup factor defined as the ratio of \textbf{Baseline} to \textbf{Current}.
    \item \textbf{Feasible}: The time required to compute a (feasible, non-optimal) separating hyperplane, in seconds.
    \item \textbf{Refinement}: The time required to iteratively optimize the initial separating hyperplane.
    \item \textbf{(Opt. time)}: The cumulative time spent computing optimal separating hyperplanes for the active sets during the iterative refinement process, included in \textbf{Current} time, in seconds.
    \item \textbf{Iters.}: The number of iterations required to reach the optimum.
    \item \textbf{Margin}: The margin of the optimal separating hyperplane (Euclidean distance to the nearest sample).
    \item \textbf{Supports}: The number of support vectors of the optimal separating hyperplane.
    \item \textbf{Actives (max)}: The average active set size over the iterations, with the maximum value shown in parentheses.    
\end{itemize}

\begin{table*}[t]
\caption{Experimental results comparing the baseline global maximum-margin solver with the proposed iterative refinement method on the MNIST dataset.}
\scriptsize
\setlength{\tabcolsep}{5.8pt}
\begin{tabular}{c*{12}{c}}
\hline

\makecell{\textbf{Exp.}} &\makecell{\textbf{Samples}} &\makecell{\textbf{Baseline}} &\makecell{\textbf{Current}} &\makecell{\textbf{Speedup}} &\makecell{\textbf{Feasible}} &\makecell{\textbf{Refinement}} &\makecell{\textbf{(Opt. time)}} &\makecell{\textbf{Iters.}} &\makecell{\textbf{Margin}} &\makecell{\textbf{Supports}} &\makecell{\textbf{Actives (max)}}\\

\hline

\makecell{test\_0\_4} &\makecell{1962} &\makecell{\textbf{1.15s}} &\makecell{6.07s} &\makecell{0.19} &\makecell{2.02s} &\makecell{4.05s} &\makecell{(3.66s)} &\makecell{104} &\makecell{208.8478} &\makecell{69} &\makecell{39.2 (72)} \\
\makecell{test\_1\_7} &\makecell{2163} &\makecell{\textbf{2.15s}} &\makecell{4.31s} &\makecell{0.50} &\makecell{1.14s} &\makecell{3.17s} &\makecell{(2.81s)} &\makecell{100} &\makecell{89.7916} &\makecell{80} &\makecell{43.7 (80)} \\
\makecell{test\_1\_9} &\makecell{2144} &\makecell{\textbf{0.94s}} &\makecell{3.86s} &\makecell{0.24} &\makecell{1.40s} &\makecell{2.46s} &\makecell{(2.16s)} &\makecell{84} &\makecell{118.4781} &\makecell{63} &\makecell{35.1 (64)} \\
\makecell{test\_7\_8} &\makecell{2002} &\makecell{\textbf{2.31s}} &\makecell{9.12s} &\makecell{0.25} &\makecell{1.51s} &\makecell{7.62s} &\makecell{(7.06s)} &\makecell{161} &\makecell{95.3113} &\makecell{111} &\makecell{61.2 (111)} \\
\makecell{test\_4\_6} &\makecell{1940} &\makecell{\textbf{1.68s}} &\makecell{7.03s} &\makecell{0.24} &\makecell{1.68s} &\makecell{5.35s} &\makecell{(4.91s)} &\makecell{133} &\makecell{95.7831} &\makecell{93} &\makecell{51.4 (94)} \\
\makecell{test\_2\_9} &\makecell{2041} &\makecell{\textbf{1.95s}} &\makecell{16.38s} &\makecell{0.12} &\makecell{1.64s} &\makecell{14.74s} &\makecell{(13.87s)} &\makecell{251} &\makecell{65.3078} &\makecell{134} &\makecell{79.9 (135)} \\
\makecell{test\_4\_7} &\makecell{2010} &\makecell{\textbf{1.94s}} &\makecell{9.00s} &\makecell{0.22} &\makecell{1.56s} &\makecell{7.44s} &\makecell{(6.85s)} &\makecell{176} &\makecell{72.3506} &\makecell{113} &\makecell{63.2 (114)} \\
\makecell{test\_2\_7} &\makecell{2060} &\makecell{\textbf{2.05s}} &\makecell{12.89s} &\makecell{0.16} &\makecell{1.86s} &\makecell{11.03s} &\makecell{(10.33s)} &\makecell{208} &\makecell{52.2019} &\makecell{153} &\makecell{81.8 (154)} \\
\makecell{test\_2\_6} &\makecell{1990} &\makecell{\textbf{1.87s}} &\makecell{16.87s} &\makecell{0.11} &\makecell{1.87s} &\makecell{15.00s} &\makecell{(14.16s)} &\makecell{253} &\makecell{55.2817} &\makecell{149} &\makecell{82.7 (149)} \\
\makecell{test\_8\_9} &\makecell{1983} &\makecell{\textbf{0.99s}} &\makecell{12.83s} &\makecell{0.08} &\makecell{1.70s} &\makecell{11.13s} &\makecell{(10.40s)} &\makecell{195} &\makecell{54.0334} &\makecell{138} &\makecell{74.1 (138)} \\

\hline 

\makecell{train\_0\_4} &\makecell{11765} &\makecell{\textbf{21.74s}} &\makecell{43.01s} &\makecell{0.51} &\makecell{17.98s} &\makecell{25.03s} &\makecell{(20.71s)} &\makecell{303} &\makecell{46.0946} &\makecell{168} &\makecell{95.6 (168)} \\
\makecell{train\_1\_7} &\makecell{13007} &\makecell{62.45s} &\makecell{\textbf{25.50s}} &\makecell{2.45} &\makecell{11.15s} &\makecell{\textbf{14.35s}} &\makecell{(10.56s)} &\makecell{246} &\makecell{23.1712} &\makecell{164} &\makecell{90.8 (165)} \\
\makecell{train\_1\_9} &\makecell{12691} &\makecell{\textbf{22.95s}} &\makecell{29.87s} &\makecell{0.77} &\makecell{10.53s} &\makecell{\textbf{19.34s}} &\makecell{(15.09s)} &\makecell{268} &\makecell{22.7511} &\makecell{179} &\makecell{102.6 (179)} \\
\makecell{train\_7\_8} &\makecell{12116} &\makecell{94.95s} &\makecell{\textbf{73.34s}} &\makecell{1.29} &\makecell{24.33s} &\makecell{\textbf{49.01s}} &\makecell{(42.86s)} &\makecell{433} &\makecell{16.0989} &\makecell{282} &\makecell{162.7 (282)} \\
\makecell{train\_4\_6} &\makecell{11760} &\makecell{\textbf{36.87s}} &\makecell{75.20s} &\makecell{0.49} &\makecell{22.72s} &\makecell{52.48s} &\makecell{(46.26s)} &\makecell{450} &\makecell{15.4355} &\makecell{305} &\makecell{168.6 (304)} \\
\makecell{train\_2\_9} &\makecell{11907} &\makecell{116.62s} &\makecell{\textbf{101.51s}} &\makecell{1.15} &\makecell{26.05s} &\makecell{\textbf{75.46s}} &\makecell{(67.38s)} &\makecell{570} &\makecell{9.6313} &\makecell{344} &\makecell{192.6 (344)} \\
\makecell{train\_4\_7} &\makecell{12107} &\makecell{\textbf{72.53s}} &\makecell{87.08s} &\makecell{0.83} &\makecell{15.84s} &\makecell{\textbf{71.24s}} &\makecell{(62.94s)} &\makecell{574} &\makecell{8.5969} &\makecell{342} &\makecell{203.9 (342)} \\
\makecell{train\_2\_7} &\makecell{12223} &\makecell{182.94s} &\makecell{\textbf{141.75s}} &\makecell{1.29} &\makecell{27.79s} &\makecell{\textbf{113.96s}} &\makecell{(102.67s)} &\makecell{778} &\makecell{5.1903} &\makecell{403} &\makecell{241.2 (403)} \\
\makecell{train\_2\_6} &\makecell{11876} &\makecell{164.34s} &\makecell{\textbf{137.52s}} &\makecell{1.20} &\makecell{23.99s} &\makecell{\textbf{113.52s}} &\makecell{(103.07s)} &\makecell{741} &\makecell{4.2232} &\makecell{412} &\makecell{245.5 (413)} \\
\makecell{train\_8\_9} &\makecell{11800} &\makecell{207.11s} &\makecell{\textbf{137.48s}} &\makecell{1.51} &\makecell{23.55s} &\makecell{\textbf{113.93s}} &\makecell{(104.24s)} &\makecell{685} &\makecell{1.7526} &\makecell{415} &\makecell{246.5 (415)} \\

\hline

\end{tabular}
\label{table1}
\end{table*}

The results show no reduction in computation time compared with the baseline approach (columns \textbf{Baseline} and \textbf{Current})  on the smaller datasets, containing approximately 2,000 samples (test set). In fact, the speedup factor of around 0.1 indicates that the proposed method requires approximately ten times more computation time in these cases. Similarly, the cumulative time spent solving the sequence of optimization subproblems (column \textbf{Opt. time}) can, in some cases, be ten times greater than the time required to solve the optimization problem directly.

Improvements become observable on the larger datasets (training set), containing approximately 12,000 samples. When the computation time required to obtain the initial separating hyperplane is included, the proposed method achieves a speedup in six out of the ten experiments (with significant improvement in some cases). When an initial separating hyperplane is assumed to be already available, and its computation time is thus excluded, the proposed method (column \textbf{Refinement}) outperforms the baseline in eight out of the ten experiments.

Notice that the proposed solution also successfully reached the optimal solution in all experiments, requiring between 84 and 778 iterations. Moreover, as expected, the active set sizes remained small compared to the total number of samples: the maximum active set size was approximately equal to the number of support vectors, while the average active set size was roughly half that value.

\section{Related Work}

The maximum-margin separating hyperplane is most commonly obtained by solving the convex quadratic optimization problem underlying the hard-margin Support Vector Machine (SVM) \cite{cortes1995support}. Since the fundamental work of Cortes and Vapnik, numerous optimization techniques have been developed for this problem, including interior-point methods, Sequential Minimal Optimization (SMO) \cite{platt1998sequential}, gradient-descent algorithms, and linear SVM solvers such as LIBLINEAR \cite{fan2008liblinear}. Although these methods differ substantially in their optimization strategies, they seek the optimum by solving the global optimization problem defined over the entire dataset.

The method proposed in this paper differs in that it does not begin from an unconstrained optimization problem but from a feasible separating hyperplane, and it then focuses exclusively on improving its margin while preserving separability throughout the optimization process. Consequently, every intermediate iterate represents a valid separating hyperplane, rather than an infeasible or partially optimized solution. 

Another distinguishing feature is the decomposition of the optimization process into a sequence of smaller subproblems. Instead of repeatedly considering the complete dataset, each iteration identifies the current active set, consisting of samples that determine the present margin, and computes the maximum-margin separator only for this subset. The obtained local optimum is then used to guide the improvement of the global separator. As the active set is typically much smaller than the full dataset, the method performs a series of optimizations on substantially reduced problems. 

Standard SVM optimization methods search for the optimum in the parameter space by minimizing a quadratic objective or maximizing its dual counterpart. In contrast, the proposed approach performs geometric transformations of the separating hyperplane and the optimization is interpreted directly in terms of geometric distances and vector orientations. Each iteration either reaches the optimum or increases the margin, consequently, progress is measured geometrically rather than by the decrease of a convex objective function.

Finally, the proposed method does not rely on a particular algorithm used to solve the active-set optimization problem. Any method capable of computing the maximum-margin separator for the current active set can be employed, including conventional SVM solvers. Therefore, the contribution of this work is not a replacement for existing SVM optimization techniques but an iterative alternative that exploits feasibility, active-set geometry, and successive margin improvement to approach the global optimum.

\section{Conclusion}\label{sec:Conclusions}
This paper presented a geometric approach to optimizing separating hyperplanes. Starting from an arbitrary separating hyperplane, it repeatedly identifies the current active set, computes its maximum-margin separating hyperplane, and incrementally re-aligns the current hyperplane while preserving linear separability and continuously increasing the margin. Unlike conventional SVM optimization methods, which solve a single but global optimization problem, the proposed approach decomposes the optimization into a sequence of smaller subproblems, defined by the active sets.

Beyond formalizing the algorithm itself, the paper established several theoretical results that provide additional insight into maximum-margin hyperplane optimization. In particular, we proved that a separating hyperplane that is optimal for its active set is also globally optimal, derived a geometric relationship between the margin and the angle to the optimal hyperplane, and presented a condition of optimality based on the intersection of the convex hulls of projected active samples.

Experimental results on linearly separable MNIST digit pairs demonstrate that the proposed method consistently finds the optimal separating hyperplane. Although it does not outperform direct optimization on smaller datasets, it becomes competitive on larger datasets. Throughout the experiments, the active sets remained small relative to the full datasets, supporting the motivation for solving a sequence of reduced optimization problems.

Some theoretical questions remained open. In particular, it has not been proven that the chosen re-alignment direction is optimal and yields the fastest possible increase of the margin. Neither that the method always terminates after a finite number of iterations. Establishing these properties, together with evaluating the method on larger datasets and in combination with alternative SVM solvers, constitutes an interesting direction for future research.

\section*{Dataset} The MNIST handwritten digits dataset, among other public sources,  is accessible from within the TensorFlow-Keras framework (\texttt{tensorflow.keras.datasets.mnist}).

For reproducibility of the results the source codes of the experiments are made publicly accessible in a GitHub repository: \href{https://github.com/ahajnal/Optimizing-Separating-Hyperplanes}{https://github.com/ahajnal/Optimizing-Separating-Hyperplanes}. 

\section*{Declaration interests}
The author declares that he has no known competing financial interests or personal relationships that could have appeared to influence the work reported in this paper.

\section*{Declaration of generative AI and AI-assisted technologies in the manuscript preparation process}
During the preparation of this work, the authors used OpenAI's GPT-5.5 large language model to improve the wording and clarity of the manuscript and to identify potential inconsistencies. After using this service, the authors carefully reviewed and edited the content as needed and take full responsibility for all aspects of the published article.

\vfill 

\bibliographystyle{IEEEtran}
\bibliography{references}  
 
\vfill
\pagebreak

\section*{Appendix}\label{sec:Appendix}

% ===========================================================

\begin{lemma}[Local--global optimality]
Let $X$ be a linearly separable dataset and $A \subseteq X$ be a subset of $X$. If a hyperplane $H_A$ is the maximum-margin separating hyperplane for $A$ and $H_A$ separates all samples of $X$, then $H_A$ is the optimal separating hyperplane for $X$.
\end{lemma}

\begin{proof}
Assume, for contradiction, that $H_A$ is the maximum-margin separating hyperplane for $A$, separates $X$, but $H_A$ is not the maximum-margin separating hyperplane for $X$, i.e., non-optimal. Since $X$ is linearly separable, there exists a maximum-margin separating hyperplane $H^*$ for $X$ whose margin $\gamma^*$ thus satisfies

\begin{equation}
\gamma^*>\gamma_A,
\end{equation}

where $\gamma_A$ is the margin of $H_A$ to $A$ .

Since $A \subseteq X$, hyperplane $H^*$ also separates all samples in $A$ with a margin $\gamma^*$, hence, $H^*$ separates $A$ with a margin strictly larger than $\gamma_A$, contradicting the assumption that $H_A$ is the maximum-margin separating hyperplane for $A$.

Therefore, $H_A$ is the optimal separating hyperplane for $X$.

Notice that global separation of $H_A$ is a necessary condition for local optimality to imply global optimality.
\end{proof}

% =======================================================

\begin{lemma}[Angle--margin relationship]
Let $X$ be a linearly separable dataset and let $H: w^Tx+b = 0$ be a separating hyperplane for $X$. Assume that $H$ is positioned equidistantly from the closest oppositely labeled samples. Let $A \subseteq X$ denote the active set of $H$, consisting of the samples lying on the positive and negative margin hyperplanes. Let $H_A^*: w^{*T}x+b^*=0$ be the maximum-margin separating hyperplane of $A$. 
Denote by $\gamma$ the margin of $H$, by $\gamma_{\scalebox{0.6}{A}}^*$ the margin of $H_A^*$ with respect to $A$, and by $\theta$ the angle between $w$ and $w^*$. Then

\[
\cos(\theta)=\frac{\gamma}{\gamma_{\scalebox{0.6}{A}}^*}.
\]
\end{lemma}

\begin{proof}
Given $m$ linearly separable samples in the active set $A=\{(x_i,y_i)\}_{i=1}^m$ where $x_i \in \mathbb{R}^d$ and $y_i\in\{−1,+1\}$, the Karush--Kuhn--Tucker (KKT) conditions guarantee the existence of Lagrange multipliers $\lambda_i$ satisfying

\begin{equation}\label{kkt1}
w_{\scalebox{0.6}{KKT}} = \sum_{i=1}^{m} \lambda_i y_i x_i,
\end{equation}

\begin{equation}\label{kkt2}
 \sum_{i=1}^{m} \lambda_i y_i=0,
\end{equation}

\begin{equation}\label{kkt3}
 \lambda_i \ge0, \quad i=1,...,m.
\end{equation}

Moreover, $\lambda_i>0$ if and only if $x_i$ is a support vector. Let $S\subseteq A$ denote the set of support vectors, and partition it into the positive and negative support sets

\[
S=S^+\cup S^-,
\]

where 

\[
S^+=\{x_i\in S: y_i=+1\}, \quad S^-=\{x_i\in S: y_i=-1\}.
\]

Since $y_i=+1$ for $S^+$ and $y_i=-1$ for $S^-$, Equation~(\ref{kkt2}) reduces to:

\begin{equation}\label{balance}
\begin{aligned}
 \sum_{p \in S^+}\lambda_p - \sum_{n \in S^-} \lambda_n = 0.
\end{aligned}
\end{equation}

Let

\begin{equation}
\begin{aligned}
 C = \sum_{p \in S^+}\lambda_p = \sum_{n \in S^-} \lambda_n .
\end{aligned}
\end{equation}

Define (normalized) coefficients

\begin{equation}
\begin{aligned}
 \mu_p=\frac{\lambda_p}{C}, \quad \nu_n=\frac{\lambda_n}{C}.
\end{aligned}
\end{equation}

yielding 

\begin{equation}
\begin{aligned}
\sum_{p \in S^+}\mu_p = \sum_{n \in S^-} \nu_n = 1.
\end{aligned}
\end{equation}

Define $P^*$ and $N^*$: 

\begin{equation}
\begin{aligned}
P^*=\sum_{p\in S^+}\mu_p p,
\qquad
N^*=\sum_{n\in S^-}\nu_n n,
\end{aligned}
\end{equation}

which are the convex combinations of the positive and negative support vectors, respectively.

Substituting the definitions of $P^*$ and $N^*$ into Equation~(\ref{kkt1}) yields

\begin{equation}
\begin{aligned}
w_{\scalebox{0.6}{KKT}} = C (P^* - N^*), \quad C>0.
\end{aligned}
\end{equation}

Since $w^*$ is the normal vector of the optimal separating hyperplane $H_A^*$ and $w_{\scalebox{0.6}{KKT}}$ is the normal vector obtained from the KKT conditions, both vectors define the same optimal hyperplane (they differ only by a nonzero scaling factor). Therefore

\begin{equation}
\begin{aligned}
 w^* \, \parallel  \,  P^* - N^*.
\end{aligned}
\end{equation}

Moreover, since $P^*$ and $N^*$ are convex combinations of positive and negative support vectors, respectively, they lie on the positive and negative margin hyperplanes of the optimal separator $H^*$. These two margin hyperplanes are parallel and separated by a distance of $2\gamma_A^*$, and as shown above, are orthogonal to $P^* - N^*$. Therefore, the distance between $P^*$ and $N^*$ equals twice the optimal margin:

\begin{equation}\label{two_gamma_star}
\begin{aligned}
\norm{P^* - N^*} = 2\gamma_{\scalebox{0.6}{A}}^*.
\end{aligned}
\end{equation}

Since all samples in the active set $A$, including the support vectors, lie on the margin hyperplane of $H$, which is orthogonal to $w$, so as any convex combination of the positive and the negative supports, respectively. Consequently, $P^*$ and $N^*$ are contained by positive and negative margin hyperplanes of $H$. These two margin hyperplanes are parallel and separated by a distance of $2\gamma$, therefore the orthogonal projection of $P^* - N^*$ onto the unit normal direction of $H$ satisfies

\begin{equation}
\begin{aligned}
   \left(\frac{w}{\norm{w}}\right)^T (P^* - N^*) = 2\gamma.
\end{aligned}
\end{equation}

The angle between $w$ and $w^*$ is $\theta$. Since $w^*$ is parallel  to $P^* - N^*$, $\angle(w, P^* - N^*)=\theta$. Expanding the inner product and using Equation \ref{two_gamma_star}, we yield

\begin{equation}
\begin{aligned}
\frac{\norm{w}\norm{P^* - N^*} \cos{\theta}}{\norm{w}} =2\gamma,
\end{aligned}
\end{equation}

and therefore

\begin{equation}
\begin{aligned}
2\gamma_{\scalebox{0.6}{A}}^* \cos{\theta} =2\gamma.
\end{aligned}
\end{equation}

Consequently,

\begin{equation}
\begin{aligned}
\cos{\theta} = \frac{\gamma}{\gamma_{\scalebox{0.6}{A}}^*}.
\end{aligned}
\end{equation}

\end{proof}

% =======================================================

\begin{lemma}[Optimality--projection intersection criterion]

Let $X$ be a linearly separable dataset, let $H: w^Tx+b = 0$ be a separating hyperplane for $X$. Assume that $H$ is positioned equidistantly from the closest oppositely labeled samples. Let $A \subseteq X$ denote the active set of $H$, consisting of the samples lying on the positive and negative margin hyperplanes. Let $A^+,A^- \subseteq A, A=A^+\cup A^-$ denote the sets of positive and negative samples of $A$, respectively. Let $\Pi_H(A^+)$ and $\Pi_H(A^-)$ denote their orthogonal projections onto $H$. Then $H$ is the optimal separating hyperplane of $X$ if and only if
\[
\operatorname{conv}(\Pi_H(A^+)) \cap \operatorname{conv}(\Pi_H(A^-))\neq\emptyset.
\]

\end{lemma}

\begin{proof}
We prove both directions.

\medskip
\noindent
($\Rightarrow$)
First, we show that optimality implies the intersection of the corresponding convex hulls. 

If $H$ is optimal, then, as shown in the proof of Lemma~2, the Karush--Kuhn--Tucker (KKT) conditions guarantee the existence of normalized coefficients $\mu_p$ and $\nu_n$ that define points $P^*$ and $N^*$ as convex combinations of the positive and negative support vectors, denoted by $S^+ \subseteq A^+$ and $S^- \subseteq A^+$, respectively, can be written as
 
\begin{equation}\label{conv_comb}
\begin{aligned}
    P^*=\sum_{p\in S^+}\mu_p p, \quad\quad N^*=\sum_{n\in S^-}\nu_n n,
    \end{aligned}    
\end{equation}
with
\begin{equation}\label{sum_to_one}
\begin{aligned}
    \sum_{p \in S^+}\mu_p = \sum_{n \in S^-} \nu_n = 1, \quad \mu_p, \nu_n \ge 0.
\end{aligned}    
\end{equation}

Optimality implies that the vector from $N^*$ to $P^*$ is parallel to the normal vector $w$ of $H$, therefore the orthogonal projections of $P^*$ and $N^*$ onto $H$ coincide. Denote their common projection by $M^*$.

The margin is equal to half the distance between $N^*$ and $P^*$. Define

\begin{equation}
\begin{aligned}
    \hat{w}=\frac{1}{2}(P^* - N^*).
\end{aligned}    
\end{equation}

Then

\begin{equation}\label{m_star}
\begin{aligned}
    M^* = P^* - \hat{w} = N^* + \hat{w}.
\end{aligned}    
\end{equation}

Using the coefficients in Equation~\ref{conv_comb} together with Equation~\ref{m_star}, we obtain

\begin{equation}\label{p_conv}
\begin{aligned}
    M^* = P^* - \hat{w} = \sum_{p\in S^+}\mu_p p - \hat{w} \cdot 1 = \sum_{p\in S^+}\mu_p p - \hat{w} \sum_{p\in S^+}\mu_p  = \sum_{p\in S^+}\mu_p (p-\hat{w}),
\end{aligned}    
\end{equation}

and similarly,

\begin{equation}\label{n_conv}
\begin{aligned}
     M^* = N^* + \hat{w} =  \sum_{n\in S^-}\nu_n n + \hat{w} \cdot 1= \sum_{n\in S^-}\nu_n n + \hat{w} \sum_{n\in S^-}\nu_n = \sum_{n\in S^-}\nu_n (n+\hat{w}).
\end{aligned}    
\end{equation}

Orthogonal projection is an affine transformation, the projected points are given by $p-\hat{w}$ for $p\in S^+ \subseteq A^+$ and $n+\hat{w}$ for $n \in S^- \subseteq A^-$, therefore, $M^*$ is a convex combination of both $\Pi_H(A^+)$ and $\Pi_H(A^-)$ by Equations~\ref{p_conv} and ~\ref{n_conv}, implying that

\begin{equation}
\begin{aligned}
    M^* \in conv(\Pi_H(A^+)) \cap conv(\Pi_H(A^-))).
\end{aligned}    
\end{equation}

Hence, 

\begin{equation}
\begin{aligned}
    conv(\Pi_H(P)) \cap conv(\Pi_H(N))) \neq \emptyset.
\end{aligned}    
\end{equation}

% ------------------------------------------

\medskip
\noindent
($\Leftarrow$)
Now, we show that the intersection of the corresponding convex hulls implies optimality.

Assume that

\begin{equation}
\begin{aligned}
    conv(\Pi_H(A^+)) \cap conv(\Pi_H(A^-)) \neq \emptyset.
\end{aligned}    
\end{equation}

Choose

\begin{equation}
\begin{aligned}
    M \in conv(\Pi_H(A^+)) \cap conv(\Pi_H(A^-)).
\end{aligned}    
\end{equation}

Then there exist coefficients

\begin{equation}
\begin{aligned}
\lambda_p \ge 0, \quad \lambda_n \ge 0,
\end{aligned}
\end{equation}

such that 

\begin{equation}
\begin{aligned}
\sum_{p \in A^+}\lambda_p = \sum_{n \in A^-} \lambda_n = 1,
\end{aligned}
\end{equation}

and 

\begin{equation}\label{M_as_conv}
\begin{aligned}
M = \sum_{p\in A^+}\lambda_p \Pi_H(p) = \sum_{n\in A^-}\lambda_n \Pi_H(n).
\end{aligned}    
\end{equation}

Let $\gamma$ denote the margin of $H$ and define

\begin{equation}
\begin{aligned}
    \hat{w}=\frac{\gamma}{\norm{w}}w.
\end{aligned}    
\end{equation}

Since every sample in the active set lies exactly at distance $\gamma$ from $H$,

\begin{equation}
\begin{aligned}
\Pi_H(p) + \hat{w}=p, \quad \Pi_H(n) - \hat{w}=n
\end{aligned}    
\end{equation}

for every 

\begin{equation}
\begin{aligned}
p \in A^+, \quad n \in A^-.
\end{aligned}    
\end{equation}

Define $P$

\begin{equation}
\begin{aligned}
P = \sum_{p\in A^+}\lambda_p p = \sum_{p\in A^+}\lambda_p (\Pi_H(p) + \hat{w}) = M + \hat{w},
\end{aligned}    
\end{equation}

and $N$

\begin{equation}
\begin{aligned}
N = \sum_{p\in A^-}\lambda_n n = \sum_{n\in A^-}\lambda_n (\Pi_H(n) - \hat{w}) = M - \hat{w}.
\end{aligned}    
\end{equation}

Therefore,

\begin{equation}
\begin{aligned}
P - N = 2\hat{w} = \frac{2\gamma}{\norm{w}}w.
\end{aligned}    
\end{equation}

Hence,

\begin{equation}
\begin{aligned}
P - N \parallel w.
\end{aligned}    
\end{equation}

On the other hand,

\begin{equation}
\begin{aligned}
P - N = \sum_{p\in A^+}\lambda_p p - \sum_{p\in A^-}\lambda_n n = \sum_{i\in A}\lambda_i y_i x_i  
\end{aligned}    
\end{equation}

where

\begin{equation}
\begin{aligned}
\lambda_i=
\begin{cases}
\lambda_p, & x_i\in A^+,\\
\lambda_n, & x_i\in A^-.
\end{cases}
\end{aligned}    
\end{equation}

Moreover,

\begin{equation}
\begin{aligned}
\sum_{i\in A}\lambda_i y_i = \sum_{p\in A^+}\lambda_p - \sum_{p\in A^-}\lambda_n = 0.
\end{aligned}    
\end{equation}

Thus the coefficients satisfy the balance condition of the dual SVM formulation, and

\begin{equation}
\begin{aligned}
\sum_{i\in A}\lambda_i y_i x_i = \frac{2\gamma}{\norm{w}}w,
\end{aligned}    
\end{equation}

which is the KKT representation of the normal vector (up to the positive scaling factor $2\gamma/\norm{w}$).

Therefore, $w$ is parallel to the normal vector of the maximum-margin separating hyperplane of $A$. Since $H$ is assumed to be positioned equidistantly between the positive and negative active samples, its bias also coincides with that of the maximum-margin separator. Hence, $H$ itself is the maximum-margin separating hyperplane of active set $A$.

Since $H$ separates the entire dataset $X$ and is the maximum-margin separating hyperplane of $A$, Lemma 1 implies that $H$ is also the maximum-margin separating hyperplane of $X$.
\end{proof}

Note that the assumption that $H$ is positioned equidistantly from the closest oppositely labeled samples is introduced only to simplify the proof. The necessary and sufficient optimality condition based on the intersection of the projected convex hulls holds even without this assumption.

\end{document}